\documentclass[sigconf]{acmart}
\usepackage{algorithm}
\usepackage{algorithmic}
\usepackage{multirow} 
\usepackage{paralist}
\usepackage{graphicx}
\usepackage{diagbox}
\usepackage{amsthm}

\AtBeginDocument{%
  }

\copyrightyear{2026}
\acmYear{2026}
\setcopyright{cc}
\setcctype{by}
\acmConference[KDD 2026] {Proceedings of the 32nd ACM SIGKDD Conference on Knowledge Discovery and Data Mining V.2}{August 9--13, 2026}{Jeju Island, Republic of Korea.}
\acmBooktitle{Proceedings of the 32nd ACM SIGKDD Conference on Knowledge Discovery and Data Mining V.2 (KDD 2026), August 9--13, 2026, Jeju Island, Republic of Korea}
\acmISBN{979-8-4007-2259-2/2026/08} 
\acmDOI{https://doi.org/10.1145/3770855.3817755}

\begin{document}

\title{Neural Scalable Symbolic Search Framework for \\ Complex Logical Queries with Multiple Free Variables}

\author{Weizhi Fei}
\orcid{0009-0000-1709-1605}
\affiliation{%
  \institution{Department of Mathematical Sciences, Tsinghua University}
  \city{Beijing}
  \state{}
  \country{China}
}
\email{fwz22@mails.tsinghua.edu.cn}

\author{Hang Yin}
\orcid{0009-0001-5264-1630}
\affiliation{
  \institution{Squarepoint Capital}
  \city{Singapore}
  \state{}
  \country{Singapore}
}
\email{yinhwx@gmail.com}

\author{Zihao Wang}
\orcid{0000-0002-3919-0396}
\authornote{Corresponding author.}
\affiliation{
  \institution{Department of Computer Science and Engineering, Hong Kong University of Science and Technology}
  \city{Hong Kong}
  \state{}
  \country{Hong Kong}
}
\email{zwanggc@connect.ust.hk}

\author{Shukai Zhao}
\orcid{0009-0001-4689-4826}
\affiliation{
  \institution{Department of Computer Sciences, University of Rochester}
  \city{New York}
  \state{}
  \country{USA}
}
\email{szhao27@ur.rochester.edu}

\author{Wei Zhang}
\orcid{0009-0004-6255-8373}
\affiliation{
  \institution{Department of Mathematical Sciences, Tsinghua University}
  \city{Beijing}
  \state{}
  \country{China}
}
\email{zhangwei.2020@tsinghua.org.cn}

\author{Yangqiu Song}
\orcid{0000-0002-7818-6090}
\affiliation{%
\institution{Department of Computer Science and Engineering, Hong Kong University of Science and Technology}
  \city{Hong Kong}
  \state{}
  \country{Hong Kong}
}
\email{yqsong@cse.ust.hk}

\begin{abstract}
Complex Query Answering (CQA) is a fundamental knowledge representation and reasoning task over incomplete knowledge graphs (KGs). Answering existential first-order queries with $k$ free variables (i.e., $\text{EFO}_k$ queries) is a crucial yet challenging problem, as it requires ranking answer tuples in $\mathcal{E}^k$, where $\mathcal{E}$ denotes the entity set of a KG. This quickly becomes intractable as $k$ grows. Consequently, existing benchmarks and methods rely on marginal rankings over individual variables; however, marginal rankings are a poor proxy for the true joint ranking of tuples. Building on neural symbolic search for $\text{EFO}_1$ queries, we propose Neural Scalable Symbolic Search (NS3), a budgeted framework that approximates joint ranking without enumerating $\mathcal{E}^k$. NS3 (i) answers marginalized sub-queries to obtain necessary candidate sets, (ii) merges multiple free variables into hypernodes whose domains are pruned and controlled by a dynamic budget $B$, and (iii) progressively reduces an $\text{EFO}_k$ query to an $\text{EFO}_{k-1}$ query over a budgeted reduced domain. Across three standard KG datasets, NS3 substantially improves joint ranking performance while retaining strong marginal accuracy. We further release a joint-ranking benchmark that extends existing $\text{EFO}_1$ datasets to $k=3$, enabling systematic evaluation of multi-variable queries. Our code is provided in \url{https://github.com/HKUST-KnowComp/NS3_KDD2026}.
\end{abstract}





\maketitle
\newcommand\kddavailabilityurl{https://doi.org/10.5281/zenodo.20355345}
\ifdefempty{\kddavailabilityurl}{}{
\begingroup\small\noindent\raggedright\textbf{Resource Availability:}\\
The source code of this paper has been made publicly available at \url{\kddavailabilityurl}.
\endgroup
}

\section{Introduction}
Knowledge Graphs (KGs) are powerful graph databases with machine-interpretable semantics, which have been widely adopted for many advanced web applications~\citep{rastogi2012building,hirsch2009interactive,remus2017storyfinder,fei2026scaling,bai2026intention,liu2022joint,liu2025hyperkgr,liu2025graph, liu2026morgan}. Despite their wide applications, current large-scale KGs are often incomplete, with substantial knowledge missing from the observed graphs~\citep{safavi2020codex, hu2020open}. Complex Query Answering (CQA) was proposed to predict \textbf{completed answers} to complex logical queries utilizing machine learning models. Traditional graph search methods overlook these potential answers, while large language models fail to provide faithful answers due to hallucinations. To facilitate interpretable reasoning on incomplete structured knowledge, CQA has emerged as a crucial task that simultaneously
addresses knowledge discovery and complex logical reasoning~\citep{ren_beta_2020,yin_rethinking_2023,ren_neural_2023}.

\begin{figure}[t]
\centering
  \includegraphics[width=\columnwidth]{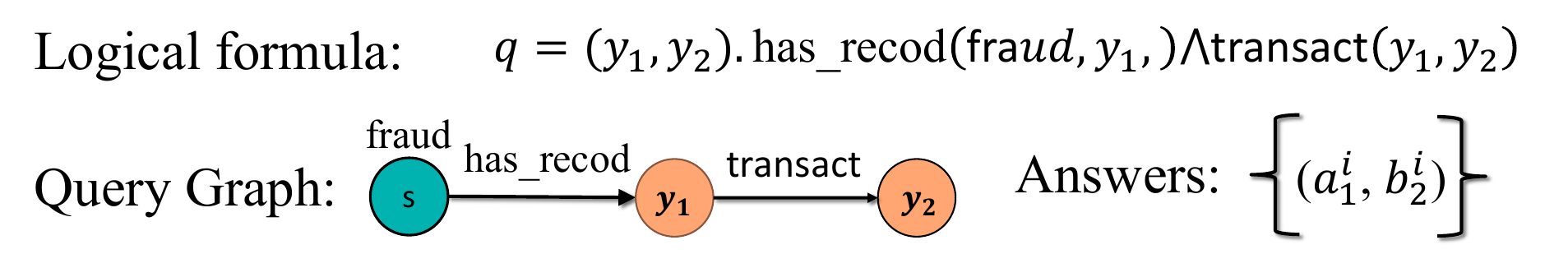}
\vspace{-2em}
  \caption{Visualization of $\text{EFO}_2$ in fraudulent activities. We present its logical formula and query graph. Notably, the answers to this query are tuples.}
  \label{fig:example}
\vspace{-1.5em}
\end{figure}

With significant advancements, the scope of complex logical queries supported by CQA models has continued to expand~\citep{hamilton_embedding_2018, ren_query2box_2020,ren_beta_2020,wang_benchmarking_2021,yin25_efok,bai2025top}. However, most existing studies have primarily focused on Existential First-Order logic queries with a single free variable ($\text{EFO}_1$ queries). The answers to these queries are single entities rather than tuples of entities. In particular, Existential First-Order logical queries with $k$ free variables ($\text{EFO}_k$ queries) not only enhance the expressive power of logical queries but also can effectively facilitate many real-world applications~\citep{ren_neural_2023}. For example, as illustrated in Figure~\ref{fig:example}, $\text{EFO}_k$ queries can be used to identify groups of individuals involved in fraudulent activities, a task that naturally arises in finance, security, and social network analysis.

Answering $\text{EFO}_k$ queries is fundamentally challenging due to the \textbf{exponential} growth of the candidate answer space. With entity set $\mathcal{E}$, a $k$-variable query has candidates in the Cartesian product $\mathcal{E}^k$. Under KG incompleteness, many tuples can be valid answers only after reasoning over missing facts, which makes exhaustive joint inference over $\mathcal{E}^k$ computationally prohibitive for realistic KGs.

Due to this difficulty, the only existing benchmark for multi-variable CQA, $\text{EFO}_k$-CQA~\citep{yin25_efok}, evaluates models primarily via \emph{marginal rankings}, where each free variable is ranked independently over $\mathcal{E}$. However, marginal rankings are a weak proxy for the true \emph{joint ranking} over tuples. For instance, if the true 2-tuple answers are $\{(1,3),(2,4)\}$, the marginal answer sets are $\{1,2\}$ and $\{3,4\}$, which cannot eliminate non-answers such as $(1,2)$ or $(2,3)$. Although $\text{EFO}_k$-CQA also proposes an approximation by sorting tuples from the Cartesian product, its empirical performance remains suboptimal~\citep{yin25_efok}. Consequently, there remains a clear gap between (i) the \emph{correct objective} of joint tuple ranking and (ii) what existing benchmarks and methods can efficiently support.

To bridge this gap, we propose \textbf{Neural Scalable Symbolic Search (NS3)}, a framework that directly infers \emph{joint rankings} for $\text{EFO}_k$ queries by progressively reducing the joint candidate space. NS3 builds on neural-symbolic search methods developed for well-studied $\text{EFO}_1$~\citep{bai_answering_2023,yin_rethinking_2023,Fei2025EfficientAS}, preserving their strengths in interpretability and strong performance while extending them to multi-variable joint inference. Our key idea is to transform an $\text{EFO}_k$ query into an equivalent $\text{EFO}_1$ query over a \emph{budgeted} reduced domain:
(i) a \textbf{marginalization} transformation generates marginal queries over subsets of free variables, whose results provide necessary conditions for pruning candidate tuples;
(ii) a \textbf{merge} transformation packages multiple free variables with pruned domains into a hypernode with a reduced domain, enabling symbolic search to operate on joint candidates;
(iii) a \textbf{budgeting} algorithm dynamically selects and merges variables, progressively reducing an $\text{EFO}_k$ query to an $\text{EFO}_1$ query under a reduced computational budget.

We evaluate NS3 on three standard KGs. On the existing $\text{EFO}_k$-CQA benchmark, we compute marginal rankings by solving the corresponding marginal queries and consistently outperform all baselines across multiple metrics. To directly assess joint inference and scalability beyond $k=2$, we further construct a concise yet scalable $\text{EFO}_k$ dataset tailored for joint ranking evaluation, introducing more challenging multi-variable query structures. Experiments on this dataset demonstrate that NS3 achieves substantial improvements in joint ranking over all baseline methods, highlighting its robustness and scalability.

\section{Related work}\label{app: related work}
Complex query answering extends KG link prediction~\citep{shi2018open,trouillon2016complex,toutanova_representing_2015} to logical queries with conjunction, disjunction, negation, existential quantification, directed graph structures, cycles, and multiple free variables~\citep{hamilton_embedding_2018,ren_query2box_2020,ren_beta_2020,he2025dage,yin_rethinking_2023,yin25_efok}. Existing work also studies CQA over temporal, event, and multi-modal KGs~\citep{fei-etal-2025-extending,lin2023tflex,bai2023complex,kharbanda2024rcone,li2026towards}.

\textbf{Neural query encoders.}
Query embedding methods encode logical queries into latent regions or distributions, including vectors~\citep{hamilton_embedding_2018}, boxes~\citep{ren_query2box_2020}, cones~\citep{zhang_cone_2021,yin_meta_2024}, Beta distributions~\citep{ren_beta_2020}, Gaussian densities~\citep{choudhary_probabilistic_2021}, and other distributions~\citep{yang_gammae_2022,wang-etal-2023-wasserstein}. GNN- and Transformer-based models further encode query graphs or linearized query structures~\citep{wang_logical_2023,zhang_clmpt_2024,bai2023sequential,zheng2025enhancing,liu2025neural}. The surprising theoretical upper bound for the query embedding method was recently revealed in~\citep {wang2026mathbb}, but gaps remain in learnability. Multi-modal distribution methods such as Query2Particles and Query2GMM improve single-variable answer distributions; Appendix~\ref{app:ns3_additional} discusses why they differ from tuple-level $\mathrm{EFO}_k$ joint ranking.

\textbf{Neural-symbolic and LLM-based methods.}
Neural-symbolic methods represent variables as fuzzy sets and apply differentiable logical operations or symbolic search~\citep{chen_fuzzy_2022,zhu_neural-symbolic_2022,xu_neural-symbolic_2022,galkin2024a,arakelyan_complex_2020,arakelyan_adapting_2023,bai_answering_2023,yin_rethinking_2023}. Recent work further accelerates symbolic search by pruning variable domains with neural scores~\citep{Fei2025EfficientAS}. LLM-based CQA methods decompose logical queries and reason over retrieved subgraphs~\citep{xuGenerateonGraphTreatLLM2024,zheng2024clr,choudharyComplexLogicalReasoning2023,xiaImprovingComplexReasoning2025,liu2026mixrag, liu2025unifying, liu2025monte}, but they are mainly designed for decomposable single-free-variable settings. In contrast, NS3 targets joint tuple ranking for $\mathrm{EFO}_k$ queries by combining marginal pruning, hypernode merging, and budgeted symbolic search.

\begin{figure*}[t]
\vspace{-0.5em}

\centering
  \includegraphics[width=\linewidth]{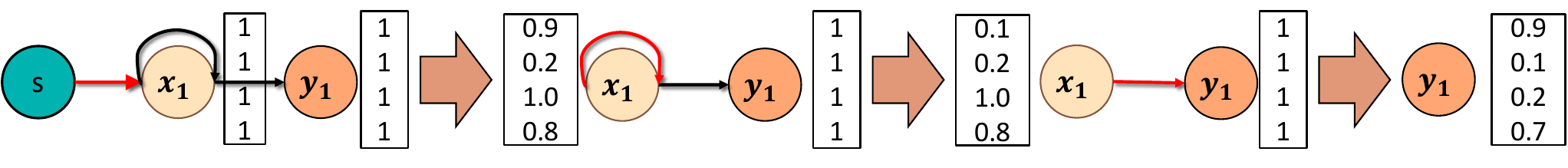}
\vspace{-2em}

  \caption{Visualization of the inferring the $\text{EFO}_1$ query with existing neural symbolic search methods~\citep{yin_rethinking_2023}. These  methods gradually remove the edges connected with constant nodes, self-loop edges, and the edges connected with leaf nodes.  The fuzzy vectors are updated accordingly, and the final fuzzy vector for the free variable can induce the predicted answer set.}
  \label{fig:visualization search}
  \vspace{-1em}

\end{figure*}

\section{Background}

\subsection{Preliminaries in $\text{EFO}_k$ queries}
Let $\mathcal{E}$ and $\mathcal{R}$ be the set of real-world entities and relations, respectively. The knowledge graph is defined  as $\mathcal{KG} \subseteq \mathcal{E} \times \mathcal{R} \times \mathcal{E}$, where each triple $(s, r, o) \in \mathcal{KG}$ indicates that the relation $r$ holds between entities $s$ and $o$. Without loss of generality, we assume the $\text{EFO}_k$ queries are expressed as the Disjunctive Normal Form (DNF) to simplify the discussion.

\begin{definition}[Atomic Formula]
Let variables $x$ and $y$ range over entities $\mathcal{E}$. An atomic formula $r(h,t) \in \{ \texttt{True}, \texttt{False}\}$ is a binary predicate, where $r \in \mathcal{R}$ and $h,t$ are either entities or variables.
\end{definition}

\begin{definition}[$\text{EFO}_k$ query] An $\text{EFO}_k$ query is a first-order logic formula involving logical connectives $\land$, $\lor$, $\lnot$, and existential quantification $\exists$, with $k$ free variables. It is defined as
\begin{align}
\psi(y_1, \cdots, y_k; x_1, \cdots, x_t) = \exists x_1, \cdots, x_t. 
\\ (c^1_1 \land \cdots \land c^1_{n}) \lor \cdots \lor (c^{\kappa}_1 \land \cdots \land c^{\kappa}_{n_{\kappa}}),
\end{align}
where $x$ and $y$ represent existential variables and free variables. The term $c$ denotes the atomic formula $r(h, t)$ or its negation $\neg r(h, t)$.
\end{definition}
\begin{definition}[Answer set of the $\text{EFO}_k$  query]
    Given an $\text{EFO}_k$ query $\psi(y_1,\cdots,y_k)$ and the corresponding knowledge graph $\mathcal{KG}$, its answer set is defined as a set of tuples:
    \begin{align*}
        \mathcal{A}[\psi(y_1,\cdots,y_k)] = \{(a_1^j,\cdots,a_k^j) |\psi(a_1^j,\cdots, a_k^j) \text{ is True in $\mathcal{KG}$}\}.
    \end{align*}
\end{definition}
Because the observed KGs are incomplete, many answers can only be predicted by machine learning methods are important. The KGs are usually split as train/valid/test nested graphs. The \textbf{hard answers} that emerged from the unseen facts in the valid/test KG are particularly important, as they can only be predicted by machine learning models. Previous $\text{EFO}_1$ queries utilize the metrics based on ranking of these \textbf{hard answers} as evaluation protocols, including MRR, and HIT$@10$\citep{ren_beta_2020}. Answering $\text{EFO}_k$ queries requires ranking over Cartesian product of size $|\mathcal{E}|^k$, which quickly becomes computationally infeasible even for moderate $k$.

\subsection{Limitations of the existing CQA methods for $\text{EFO}_k$ queries}
While existing CQA methods are effective  for $\text{EFO}_1$ queries, extending them to $\text{EFO}_k$ queries introduces significant computational challenges.  The joint answer space grows exponentially with the number of free variables $k$, making it infeasible for existing methods to score all candidate tuples. This challenge is particularly pronounced for query embedding methods, which rely on large-scale training over sampled queries.
As a result, standard $\text{EFO}_k$ benchmarks (e.g., $\text{EFO}_k$-CQA) evaluate models using marginal rankings rather than joint rankings.
However, marginal ranking is insufficient for representing joint ranking of multi-variable answers. Although the joint metrics propose one closed solution to estimate joint rankings by marginal rankings, the empirical results are suboptimal.

\subsection{Neural symbolic search methods for the $\text{EFO}_1$ query}
Neural symbolic search methods employ knowledge graph embeddings as neural link predictors to infer unobserved facts. By leveraging fuzzy logic~\citep{mendel1995fuzzy} to relax logical operators, these methods show that answering $\text{EFO}_1$ queries is equivalent to performing a sequence of edge-removal operations on the query graph. Each variable node is associated with a fuzzy vector representing its soft assignments over entities. As edges are progressively removed and logical constraints are propagated, the query graph is iteratively simplified until only the free variable remains. The fuzzy vector of the free variable then induces the predicted answer set. Figure~\ref{fig:visualization search} illustrates this execution process, including the elimination of constant nodes, self-loop edges, and leaf nodes.

\begin{definition}[Product norm]
Fuzzy logic is employed to execute the fuzzy logical operations replacing the logical operations.  We provide the widely used product norm as follows, where Conjunction: $\alpha ~\top_{P}~ \beta = \alpha * \beta$, Negation: $1 - \alpha$, and Disjunction: $\alpha ~\bot_{P}~ \beta =  1 - (1-\alpha) \top_{P} (1-\beta)$.
\end{definition}

\begin{definition}[Query graph]
Let $\phi$ be a conjunctive query, its query graph $G_{\phi} = \{(h_i, r_i, t_i, \textsc{Neg}_i)\}$ consists of quadruples, where each quadruple corresponds to an atomic formula or its negation. This representation corresponds to an edge with two endpoints $h$ and $t$, along with two attributes: $r$, which denotes the relation, and $\textsc{Neg}_i$, a boolean variable indicating whether the atom is positive.
\end{definition}

\begin{definition}[Fuzzy vector]
If we index the entities in $\mathcal{E}=\{e_1,\ldots,e_{|\mathcal{E}|}\}$, we represent a \emph{fuzzy set over entities} for each variable $x$ using a \emph{fuzzy vector} $C_x \in [0,1]^{|\mathcal{E}|}$, where the $i$-th entry $C_x[i]$ denotes the membership degree of entity $e_i$ being a valid assignment for $x$ under current constraints. We define $\mu(s, C_x)$ to retrieve the membership degree of entity $s$ for variable~$x$.
\end{definition}


\begin{figure*}[t]
\centering
  \includegraphics[width=\linewidth]{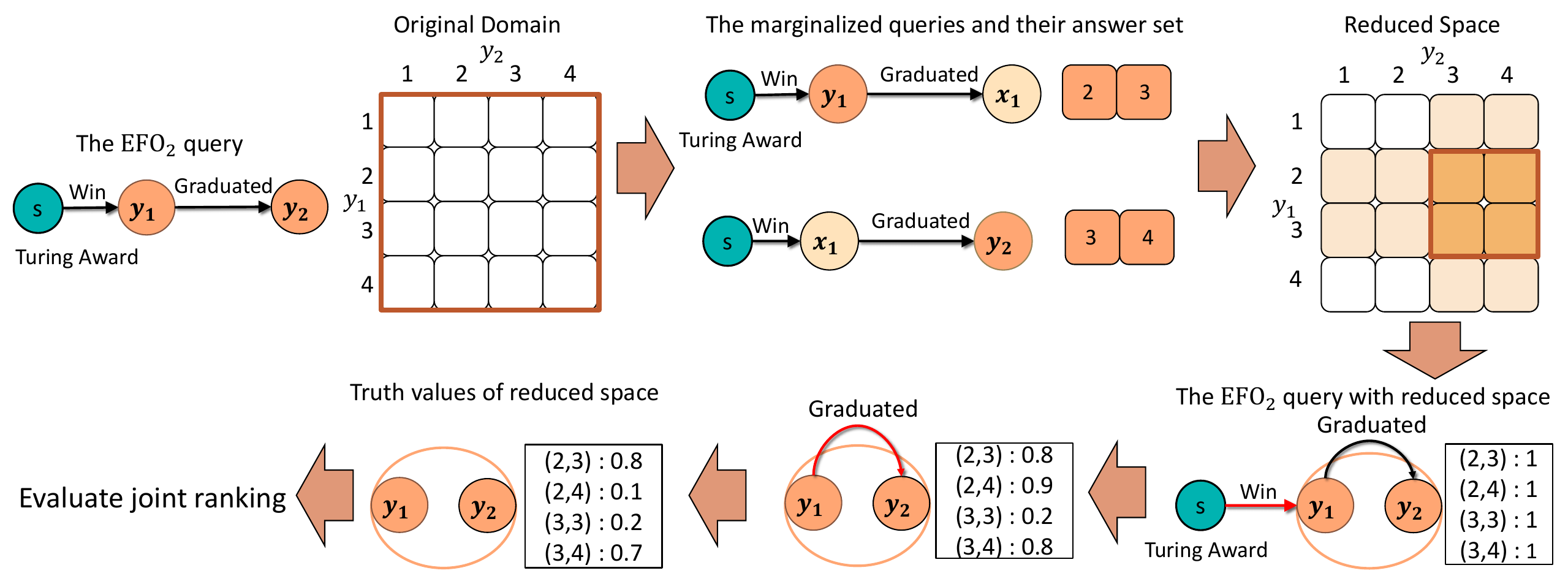}
  \vspace{-0.5em}
  \caption{Visualization to infer the joint ranking of $\text{EFO}_2$ query by the marginalization and merge transformation. First, we infer the marginal answer sets on two marginal queries. Second, we merge the $2$ free variables as one hypernode and reduce the search space by marginal answers. Finally, we adapt the symbolic search method to infer this new query on the reduced domain, thereby resulting in joint ranking.}
  \label{fig:reduces and trans}
\vspace{-1em}
\end{figure*}

\section{Methodology}

In this section, we present \textbf{Neural Scalable Symbolic Search (NS3)}, a framework for answering existential first-order queries with multiple free variables ($\mathrm{EFO}_k$) over incomplete knowledge graphs.
NS3 extends neural symbolic search methods originally developed for $\mathrm{EFO}_1$ queries to the multi-variable setting by explicitly controlling the joint search space.
Rather than enumerating the exponential candidate space $\mathcal{E}^k$, NS3 performs a \emph{budgeted joint search} through two key steps:
(1) identifying necessary candidate sets via marginalization, and
(2) progressively merging free variables into hypernodes with bounded joint domains.
An overview of the framework is illustrated in Figure~\ref{fig:reduces and trans}.
\subsection{Marginalization as a Necessary Condition}

We begin by formalizing marginalization, which provides a necessary (but not sufficient) condition for valid joint answers and serves as the basis for safe candidate pruning.

\paragraph{Marginalization.}
Any valid joint answer tuple to an $\mathrm{EFO}_k$ query must project to a valid answer of each corresponding marginal query obtained by existentially quantifying a subset of free variables. Therefore, candidate assignments that violate any marginal constraint can be safely pruned. However, satisfying all marginal constraints does not guarantee joint validity, which motivates the subsequent joint reasoning steps in NS3.

\begin{definition}[Marginalization of an $\mathrm{EFO}_k$ Query]
Given an $\mathrm{EFO}_k$ query $\phi(\mathcal{Y})$ with free variables
$\mathcal{Y}=\{y_1,\ldots,y_k\}$ and a subset
$\mathcal{Y}^{\delta}\subsetneq\mathcal{Y}$,
the marginalization of $\phi$ over $\mathcal{Y}^{\delta}$ is defined as
\begin{equation}
\mathcal{M}[\phi,\mathcal{Y}^{\delta}] = \phi(\mathcal{Y}^{\delta}),
\end{equation}
where free variables not in $\mathcal{Y}^{\delta}$ are transformed into existentially quantified variables.
\end{definition}

\begin{definition}[Marginalization on Answer Sets]
Let the answer set of an $\mathrm{EFO}_k$ query be
$\mathbf{a}=\{(a_1^j,\ldots,a_k^j)\}$,
and let $\mathcal{Y}^{\gamma}=\{y_{i_1},\ldots,y_{i_\tau}\}$ be a subset of free variables.
The marginalization of $\mathbf{a}$ over $\mathcal{Y}^{\gamma}$ is defined as
\begin{equation}
\mathcal{M}[\mathbf{a},\mathcal{Y}^{\gamma}]
=
\{(a_{i_1}^j,\ldots,a_{i_\tau}^j)\}.
\end{equation}
\end{definition}

A formal proof of the commutativity between marginalization and answer sets is provided in Appendix~\ref{app: proof}. This result ensures that marginal answer sets can be obtained by directly solving marginalized queries using existing symbolic search methods. Unlike independently scoring each free variable, the marginalized query preserves cross-variable constraints by existentially quantifying the other free variables.

\subsection{Merge Transformation}

Marginalization reduces the number of free variables but does not eliminate the need to reason over joint assignments. To enable joint reasoning while reusing $\mathrm{EFO}_1$ symbolic search procedures, we introduce a \emph{merge transformation} that groups multiple free variables into a single hypernode with a reduced joint domain.

\paragraph{Intuition.}
The merge transformation converts a multi-variable query into an $\mathrm{EFO}_1$ query over a reduced joint domain.
This reduction is approximate: it preserves high-confidence joint candidates while discarding low-probability regions of the Cartesian product.

\begin{definition}[Merge Transformation]
Let $G_{\phi}$ be the query graph of an $\mathrm{EFO}_k$ query $\phi$, and let $y_{i_1}$ and $y_{i_2}$ be two free variables.
The merge transformation replaces $y_{i_1}$ and $y_{i_2}$ with a hypernode $\gamma$.
All edges incident to $y_{i_1}$ or $y_{i_2}$ are reconnected to $\gamma$, whose domain represents joint assignments of $(y_{i_1},y_{i_2})$.
\end{definition}

\subsection{Reduced Joint Domain Construction}

We now describe how to construct a reduced joint domain for a merged hypernode. For clarity, we first consider merging two free variables. Let $\mathcal{M}[\phi,\{y_1\}]$ and $\mathcal{M}[\phi,\{y_2\}]$ denote the marginal queries for $y_1$ and $y_2$, producing fuzzy vectors $C_{y_1}$ and $C_{y_2}$.
Let $B$ be the per-variable budget. We allocate a reduced joint domain of size $2B \ll |\mathcal{E}|^2$.

To estimate effective marginal domain sizes, we compute fuzzy counts
\begin{equation}
C_1 = \sum_i C_{y_1}[i], \quad
C_2 = \sum_i C_{y_2}[i].
\end{equation}

\paragraph{Budget Allocation Heuristic.}
We allocate budgets using
\begin{equation}
\lambda = \sqrt{\frac{2B}{C_1 C_2}}, \quad
b_1 = \lambda C_1, \quad
b_2 = \lambda C_2,
\end{equation}
such that $b_1 b_2 \approx 2B$. This heuristic ensures that the overall joint budget is respected while allocating more capacity to variables with diffuse marginal distributions and prioritizing highly constrained variables. We then select the top-$b_1$ entities for $y_1$ and the top-$b_2$ entities for $y_2$, and form the reduced joint domain via their Cartesian product. Padding is applied when necessary to meet the budget constraint. The budget parameter $B$ therefore controls a trade-off between recall and efficiency, which we empirically analyze in Section~\ref{sec:ablation} and Appendix~\ref{app:ns3_additional}.

\subsection{Inferring $\mathrm{EFO}_2$ Queries via Marginalization and Merge}

After merging two free variables into a hypernode, the resulting query can be treated as an $\mathrm{EFO}_1$ query over the reduced joint domain.
Two modifications are required to adapt existing symbolic search procedures. First, edges between the merged variables become self-loop edges on the hypernode.
Second, when removing an edge incident to the hypernode, its fuzzy vector is updated via \emph{constraint slicing}:
each symbolic constraint induces a mask over the joint domain, and the fuzzy vector is updated by element-wise multiplication with this mask, followed by normalization.
This operation propagates symbolic constraints over joint assignments without explicitly enumerating $\mathcal{E}^2$.

\subsection{Progressive Budgeting for $\mathrm{EFO}_k$ Queries}

For general $\mathrm{EFO}_k$ queries, NS3 progressively applies merge transformations until all free variables are consolidated into a single hypernode, yielding an $\mathrm{EFO}_1$ query. If a hypernode $\gamma$ represents $n_{\gamma}$ free variables, its budget is set to
$B_{\gamma} = n_{\gamma} \cdot B$, ensuring linear growth with respect to the number of variables.

\paragraph{Budgeting Algorithm.}
We first compute fuzzy vectors for all single-variable marginal queries. At each step, we greedily merge the pair of variables (or hypernodes) whose estimated marginal domain sizes yield the smallest product. This strategy prioritizes merging the most constrained variables first, thereby limiting the growth of intermediate joint domains. After each merge, we estimate the fuzzy count of the resulting hypernode using symbolic search. The process repeats until a single hypernode remains, at which point joint ranking is computed via $\mathrm{EFO}_1$ symbolic search.

\subsection{Complexity Analysis}

The space complexity of NS3 is dominated by the neural link predictor, requiring $\mathcal{O}((|\mathcal{E}|+|\mathcal{R}|)d)$ memory, where $d$ is the embedding dimension.NS3 avoids storing dense adjacency tensors of size $\mathcal{O}(|\mathcal{E}|^2|\mathcal{R}|)$. For time complexity, let $n$ be the number of atomic formulas and $B$ the per-variable budget.
Computing marginal queries requires $\mathcal{O}(nkB^2)$ time.
At the $i$-th merge step, symbolic inference over the merged hypernode costs $\mathcal{O}(n(iB)^2)$.
Summing over all merge steps yields an overall complexity of $\mathcal{O}(nk^3B^2)$.

While this complexity grows polynomially with $k$ for fixed $B$, NS3 achieves practical scalability by operating with small $k$ and moderate budgets, as validated empirically in Section~\ref{sec: experiment}.

\section{Experiments}\label{sec: experiment}

\begin{table}[t]
    \centering
    \caption{The statistics for the three knowledge graphs are provided, including the number of entities, relations, and edges. Additionally, we present the division of the training, validation, and test graphs.}\label{tab:KGs}
    \resizebox{0.98\linewidth}{!}{
    \begin{tabular}{lccccccc}
        \toprule
        Dataset & Entities & Relations & Training Edges & Validation Edges & Test Edges \\
        \midrule
        FB15k & 14,951 & 1,345 & 483,142 & 50,000 & 59,071  \\
        FB15k-237 & 14,505 & 237 & 272,115 & 17,526 & 20,438 \\
        NELL & 63,361 & 200 & 114,213 & 14,324 & 14,267  \\
        \bottomrule
    \end{tabular}}
\vspace{-1em}

\end{table}

To validate the effectiveness of NS3 in addressing $\text{EFO}_k$ queries, we not only evaluate the existing $\text{EFO}_k$-CQA benchmark involving $2$ free variables, but also propose a new $\text{EFO}_k$ benchmark that is concise and scalable to include more free variables. These benchmarks are constructed over three standard knowledge graphs: FB15k \citep{bordes_translating_2013}, FB15k-237 \citep{toutanova_representing_2015}, and NELL995 \citep{xiong2017deeppath}, with statistics provided in Table~\ref{tab:KGs}.To study the sensitivity and effectiveness of domain budgeting, we vary the budget size and analyze its effects. Additionally, to demonstrate the effectiveness of our proposed budgeting algorithm, we conduct an ablation study comparing our approach to directly merging all free variables at once.

\subsection{Implementation details}
We adapt the neural symbolic search method NLISA~\citep{Fei2025EfficientAS} to implement our proposed NS3 framework. NS3 (J) infers the joint ranking, while NS3 (M) infers the marginal ranking by considering the marginal queries derived from the marginalization process. Since our framework only generates truth values for the reduced domain, the discarded tuples/entities that belong to the answer set are assigned a score of zero. Specifically, we set the budget $B$ to 4,000 for FB15K (14,951 entities), 4,000 for FB15K-237 (14,505 entities), and 6,000 for NELL (60,000 entities). We provide further details about the implementation and the running time of our method in Appendix~\ref{app: imple}.
\begin{table*}[ht]
\centering
\caption{HIT@10 scores(\%) of three metrics for answering queries with $2$ free variables on FB15k-237. $e$ is the number of existential variables.  We use SDAG, Multi, and Cyclic to denote the simple directed acyclic graph, multigraph, and cyclic graph, which represents the topology of query graph. NS3 (M) represents the results of marginal ranking in our framework.}
\label{tab: EFO2 result}
\footnotesize
\begin{tabular}{ccccccccccrc}
\toprule
\multirow{2}{*}{Model}  & \multirow{2}{*}{\shortstack[c]{\\ Type}}  & \multicolumn{2}{c}{$e=0$} & \multicolumn{3}{c}{$e=1$} & \multicolumn{3}{c}{$e=2$} & \multirow{2}{*}{Avg} \\
\cmidrule(lr){3-4} \cmidrule(lr){5-7} \cmidrule(lr){8-10} & & S      & M &  S  & M & C &  S  & M & C & \\
\midrule
\multirow{3}{*}{BetaE} 
& Marg&36.8&37.5&34.3&35.3&49.7&27.1&25.8&37.7&33.0 \\
& Mult&28.9&25.3&24.2&22.5&32.6&23.9&22.0&26.4&24.1\\
& Joint&5.9&5.7&5.1&4.9&12.3&2.2&2.1&6.6&4.6\\
\midrule
\multirow{3}{*}{LogicE} 
& Marg&39.7&38.8&36.2&36.5&51.1&27.8&26.2&38.3&34.0 \\
& Mult&31.9&27.3&26.4&24.3&34.5&25.1&23.2&26.5&25.4\\
& Joint&6.4&6.3&5.7&5.4&13.6&2.6&2.3&7.0&5.0 \\
\midrule
\multirow{3}{*}{ConE} 
& Marg&40.1&40.8&37.6&38.1&55.0&28.8&27.5&41.4&35.8 \\
& Mult&32.8&28.9&27.7&25.3&37.9&26.0&23.9&30.4&27.0\\
& Joint&6.7&7.2&6.0&5.8&14.2&2.7&2.4&7.6&5.4 \\
\midrule
\multirow{3}{*}{CQD} 
& Marg&36.7&34.3&35.6&36.4&49.6&25.5&24.0&39.3&33.0 \\
& Mult  &34.6&28.7&29.8&26.9&36.8&28.8&25.0&28.8&27.9\\
& Joint   &6.4&6.9&6.3&6.4&13.9&3.0&2.7&7.6&5.6  \\
\midrule
\multirow{3}{*}{lmpnn} 
& Marg&40.9&40.0&38.5&38.4&52.7&30.3&27.0&39.0&35.4 \\
& Mult  &34.6&28.7&29.8&26.9&36.8&28.8&25.0&28.8&27.9\\
& Joint   &6.8&7.5&5.9&5.6&13.1&2.6&2.3&6.4&5.0  \\
\midrule
\multirow{3}{*}{FIT} 
& Marg&45.9&43.1&45.0&45.4&50.2&36.4&34.8&42.0&41.2 \\
& Mult  &40.7&34.3&37.2&34.4&33.7&36.2&34.7&33.7&\textbf{35.0}\\
& Joint    &6.7&7.0&7.1&7.2&11.7&3.2&3.4&7.1&5.9   \\
\midrule
\multirow{3}{*}{NS3M} 
& Marg&48.0&46.7&46.0&46.1&59.0&35.5&33.2&48.7&\textbf{42.8} \\
& Mult  &40.1&33.9&36.9&33.3&40.2&34.6&31.1&39.1&34.8\\
& Joint    &6.9&7.6&7.4&7.7&14.6&3.3&3.3&9.0&\textbf{6.6}   \\
\bottomrule
\end{tabular}
\end{table*}
\subsection{Evaluation}

\begin{definition}[Joint ranking]
    Given a query $\phi(y_1, \cdots, y_k)$, the space of candidate tuples is $\mathcal{E}^k$. The joint ranking provides the rank for each tuple in relation to all the candidates.
\end{definition}

\begin{definition}[Marginal ranking]
    Given a query $\phi(y_1, \cdots, y_k)$, the space of candidates for $i$-th free variable is $\mathcal{E}$. The marginal ranking provides the rank for each candidate entity.
\end{definition}

Based on marginal ranking, the $\text{EFO}_k$-CQA benchmark~\citep{yin25_efok} introduces three families of evaluation metrics: Marginal, Multiply, and Joint. These metrics aim to approximate the joint ranking of answers through marginal rankings, with increasing levels of difficulty. Among them, the Joint is the most challenging, as it estimates the joint ranking via combinatorial aggregation. Details of the evaluation procedure are provided in Appendix~\ref{app: eval}. Since existing methods typically achieve low rankings on $\text{EFO}_k$ queries~\citep{yin25_efok}, we report HIT@10.

\subsection{Baselines}
We consider six representative CQA models as baselines\footnote{We exclude NLISA~\citep{Fei2025EfficientAS} from the baselines because it usually has weaker performance compared with FIT.} to compare. For query embedding methods, we include BetaE~\citep{ren_beta_2020}, LogicE~\citep{luus_logic_2021}, and ConE~\citep{zhang_cone_2021}. For symbolic search methods, we select CQD~\citep{arakelyan_complex_2020} and FIT~\citep{yin_rethinking_2023}. Additionally, we include the graph neural network method LMPNN~\citep{wang_logical_2023}. Notably, FIT is a previous state-of-the-art model across various CQA benchmarks.  We follow the approach of \citet{yin25_efok} to adapt the implementation of existing CQA models, enabling them to marginally infer $\text{EFO}_k$ queries; full details are provided in Appendix~\ref{app: baselines}. We exclude LLM-based methods mainly for structural reasons: existing LLM-based CQA methods typically rely on tree decomposition for DAG-structured $\mathrm{EFO}_1$ queries, while general $\mathrm{EFO}_k$ queries require tuple-level reasoning over coupled free variables.

\subsection{$\text{EFO}_k$-CQA, the existing $\text{EFO}_k$ queries benchmark}\label{sec: efok cqa}

We first consider the existing benchmark for evaluating logical queries involving multiple free variables,  $\text{EFO}_k$-CQA benchmark~\citep{yin25_efok}. The $\text{EFO}_k$-CQA benchmark includes a total of 741 different query types, with a maximum of $2$ free variables. This benchmark is developed by enumerating valid query types and aims to provide a comprehensive evaluation, encompassing three topological structures: Simple Directed Acyclic Graphs (SDAG), multi-edge graphs, and cyclic graphs. The combinatorial space of query types is parameterized by the number of constants, existential variables, free variables, and the number of edges. We utilize three kinds of metrics to evaluate the marginal ranking, where the details of Marginal, Multiply, and Joint metrics are referred to Appendix~\ref{app: eval}. Due to the page limit, we also follow the $\text{EFO}_k$-CQA benchmarks to only conduct the experiments on FB15k-237~\citep{toutanova_observed_2015} and we only evaluate the queries involving two variables in the $\text{EFO}_k$-CQA benchmark, shown in Table~\ref{tab: EFO2 result}. The additional experimental results can be referred to the Appendix~\ref{app: further exper}.

Based on the results in Table~\ref{tab: EFO2 result}, we have the following empirical observation:
\begin{enumerate}
    \item \textbf{NS3 (M) surpasses all the baseline models on the average results over two kinds of metrics, including marginal and joint metrics.}, which validate the effectiveness of using the symbolic search method to infer marginal queries for free variables separately. Additionally, our method performs exceptionally well on the challenging cyclic queries. This demonstrates the effectiveness of the marginalization transformation, which is capable of estimating the marginal answer set. The reason NS3 (M) surpasses FIT is that FIT does not account for the propagation of information between free variables, whereas our marginal operation addresses this gap.
    \item \textbf{The joint metrics are significantly lower than the results of the marginal and multiplicative metrics across all the models and queries.} This phenomenon suggests that the gap in using marginal ranking to estimate joint ranking is amplified, leading to increased errors in marginal ranking. This highlights the limitations of marginal ranking and underscores the necessity of directly solving joint ranking.
\end{enumerate}

\subsection{More advanced $\text{EFO}_k$ benchmark}

\begin{figure*}[ht]

\centering
  \includegraphics[width=\linewidth]{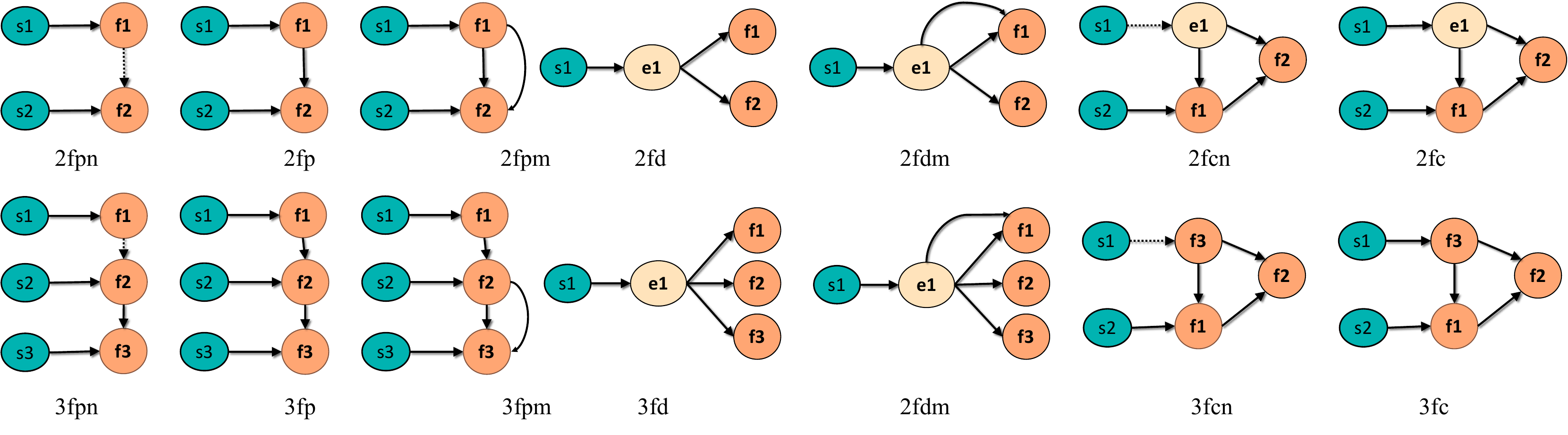}
\vspace{-2em}

  \caption{Visualization of the query types in our scalable $\text{EFO}_k$ dataset by query graph. We use ``s'', ``e'',  and ``f'' to represent the constant variable, existential variable, and free variable, respectively. The black edges represent positive atomic formulas, while the red edges indicate atomic formulas with negation. To distinguish the previously queries, we added "f" to their names.}
  \label{fig:query types}
\vspace{-1em}
\end{figure*}

\begin{table*}[ht]
\centering
\caption{HIT@10 scores(\%) for answering queries with multiple free variables across three standard KGs. Avg($k=2$) and Avg($k=3$) represent the average scores for queries involving $k=2$ and $k=3$ free variables, respectively. The best results are highlighted by \textbf{bold}, while the second-best results are indicated with \underline{underline}. }
\label{tab: EFO3 result}
\resizebox{\textwidth}{!}{

    \begin{tabular}{ccccccccccccccccccc}
    \toprule
    KGs &models & 2fpn & 2fp & 2fpm & 2fd & 2fdm  & 2fcn & 2fc & 2fpn & 3fp & 3fpm & 3fd & 3fdm & 3fcn  & 3fc & Avg.($2$)  & Avg.($3$)  & Avg
    \\ \hline
        \multirow{8}{*}{FB15k-237}&BetaE & 4.1  & 9.0  & 6.7  & 3.7  & 2.2  & 17.4  & 3.9  & 1.9  & 4.0  & 4.4  & 2.4  & 2.3  & 1.6  & 0.3  & 6.7  & 2.4  & 4.6   \\ 
        &LogicE & 3.3  & 9.5  & 7.0  & 5.8  & 2.8  & 18.0  & 3.7  & 1.9  & 5.2  & 5.5  & 2.6  & 2.3  & 1.2  & 0.5  & 7.2  & 2.7  & 5.0   \\ 
        &ConE & 4.4  & 9.3  & 7.6  & 5.3  & 2.8  & 17.7  & 9.2  & 2.6  & 5.4  & 5.4  & 2.7  & 2.7  & 0.5  & 0.3  & 8.0  & 2.8  & 5.4   \\ 
        &LMPNN & 2.6  & 9.7  & 8.1  & 4.0  & 1.9  & 15.3  & 6.9  & 1.7  & 6.0  & 7.3  & 2.2  & 2.5  & 7.3  & 1.9  & 6.9  & 4.1  & 5.5   \\ 
        &CQD & 1.8  & 9.5  & 9.2  & 4.1  & 2.2  & 16.1  & 4.0  & 0.7  & 8.3  & 8.6  & 2.2  & 2.3  & \underline{9.9}  & 0.0  & 6.7  & 4.6  & 5.6   \\ 
        &FIT & 3.3  & 9.7  & 7.7  & 5.7  & 3.2  & \underline{14.9}  & \underline{10.1}  & 2.9  & 8.5  & 6.7  & 3.1  & 4.2  & 5.9  & {3.5}  & {7.8}  & {5.0}  & {6.4}   \\ 
        &NS3 (M) & \underline{4.5}  & \underline{14.4}  & \underline{14.1}  & \underline{6.0}  & \underline{3.9}  & 13.7  & 6.0  & \underline{5.2}  & \underline{14.9}  & \underline{17.9}  & \underline{3.7}  & \underline{5.2}  & \underline{7.3}  & \underline{3.1}  & \underline{8.9}  & \underline{8.2}  & \underline{8.6}   \\ 
        &NS3 (J) & \textbf{11.4}  & \textbf{28.5}  & \textbf{28.3}  & \textbf{16.2}  & \textbf{10.0}  & \textbf{36.5}  & \textbf{19.2}  & \textbf{8.4}  & \textbf{19.4}  & \textbf{26.2}  & \textbf{8.0}  & \textbf{8.4}  & \textbf{17.0}  & \textbf{8.9}  & \textbf{21.4}  & \textbf{13.8}  & \textbf{17.6}   \\ 
        \midrule
        \multirow{8}{*}{FB15K}& BetaE & 9.6  & 15.8  & 18.0  & 9.7  & 9.6  & 11.0  & 25.4  & 3.7  & 7.2  & 9.9  & 4.5  & 5.7  & 0.5  & 1.6  & 14.2  & 4.7  & 9.4   \\ 
        &LogicE & 11.0  & 17.0  & 20.6  & 12.3  & 11.8  & 10.5  & 25.3  & 3.9  & 7.4  & 11.7  & 6.3  & 7.6  & 0.4  & 2.1  & 15.5  & 5.6  & 10.6   \\ 
        &ConE & 10.6  & 18.5  & 21.0  & 11.8  & 12.4  & 16.3  & 26.6  & 6.2  & 9.6  & 13.1  & 6.6  & 8.2  & 0.6  & 0.2  & 16.7  & 6.4  & 11.6   \\ 
        &LMPNN & 7.1  & 17.5  & 13.6  & 9.3  & 7.5  & 17.1  & 19.5  & 7.8  & 17.9  & 17.9  & 8.8  & 8.5  & 6.0  & 14.7  & 13.1  & 11.7  & 12.4   \\ 
        &CQD & 4.2  & 21.3  & 36.1  & 17.4  & 13.7  & 12.3  & 23.6  & 1.6  & 21.1  & 19.1  & 8.9  & 9.6  & 0.9  & 11.9  & 18.4  & 10.4  & 14.4   \\ 
        &FIT & 10.7  & 19.7  & 24.3  & 13.3  & 14.6  & 19.9  & 19.9  & 5.3  & 14.5  & 18.0  & 9.0  & 10.3  & \underline{8.1}  & 9.5  & 17.5  & 10.7  & 14.1   \\ 
        &NS3 (M) & \underline{11.3}  & \underline{28.0}  & \underline{36.0}  & \underline{17.5}  & \underline{20.1}  & \underline{32.9}  & \underline{24.5}  & \underline{11.0}  & \underline{27.0}  & \underline{32.6}  & \underline{12.0}  & \underline{15.0}  & 3.3  & \underline{15.0}  & \underline{24.3}  & \underline{16.6}  & \underline{20.4}   \\ 
        &NS3 (J) & \textbf{40.3}  & \textbf{61.3}  & \textbf{67.2}  & \textbf{61.6}  & \textbf{64.5}  & \textbf{47.2}  & \textbf{65.7}  & \textbf{30.2}  & \textbf{45.0}  & \textbf{48.4}  & \textbf{35.6}  & \textbf{39.7}  & \textbf{17.2}  & \textbf{37.2} & \textbf{58.3}  & \textbf{36.2}  & \textbf{47.2}  \\ 
        \midrule
        \multirow{8}{*}{NELL}&BetaE & 6.6  & 10.0  & 22.2  & 11.8  & 9.8  & 0.6  & 30.6  & 4.4  & 12.0  & 23.9  & 12.5  & 21.3  & 0.4  & 3.0  & 13.1  & 11.1  & 12.1   \\ 
        &LogicE & 6.8  & 10.0  & 21.5  & 16.8  & 11.2  & 0.7  & 31.9  & 5.1  & 13.3  & 25.3  & 17.5  & 28.1  & 1.7  & 2.8  & 14.1  & 13.4  & 13.8   \\ 
        &ConE & 6.8  & 9.4  & 21.5  & 12.6  & 11.1  & 12.0  & 32.4  & 4.8  & 11.6  & 23.2  & 12.1  & 20.9  & 7.7  & 2.6  & 15.1  & 11.8  & 13.5   \\ 
        &LMPNN & 4.4  & 10.7  & 17.6  & 12.1  & 8.9  & 14.6  & 33.3  & 4.6  & 13.9  & 22.1  & 12.3  & 19.6  & 5.7  & 12.2  & 14.5  & 12.9  & 13.7   \\ 
        &CQD & 4.2  & 9.0  & 26.4  & 11.5  & 8.7  & 8.7  & 21.8  & 0.1  & 16.3  & 27.9  & 14.1  & 24.3  & 0.0  & 12.4  & 12.9  & 13.6  & 13.2   \\ 
        &FIT & 8.8  & 11.9  & 24.0  & 17.1  & 11.6  & \underline{12.9}  & \underline{29.1}  & 5.7  & 16.4  & 27.2  & 18.6  & \underline{31.2}  & \underline{7.7}  & \underline{9.5}  & 16.6  & 16.5  & 16.6   \\ 
        &NS3 (M) & \underline{9.5}  & \underline{13.8}  & \underline{30.5}  & \underline{16.9}  & 13.5  & \underline{8.5}  & \underline{24.4}  & \underline{8.9}  &\underline{ 20.1}  & \underline{34.4}  & \underline{19.5}  & 30.9  & 4.8  & 8.4  & \underline{16.7}  & \underline{18.1}  & \underline{17.4}   \\ 
        &NS3 (J) & \textbf{15.8}  & \textbf{32.8}  & \textbf{53.9}  & \textbf{26.9}  & \textbf{21.9}  & \textbf{24.2}  & \textbf{50.2}  & \textbf{14.7}  & \textbf{25.3}  & \textbf{41.0}  & \textbf{21.4}  & \textbf{31.6}  & \textbf{11.7}  & \textbf{21.0}  & \textbf{32.2}  & \textbf{23.8}  & \textbf{28.0}  \\ 
        \bottomrule 

    \end{tabular}

}
\vspace{-1em}
\end{table*}

The $\text{EFO}_k$-CQA benchmark currently supports queries with at most two free variables, and extending it to larger $k$ is challenging due to the exponential growth of valid query types. To systematically examine how query structures evolve with additional free variables, we construct a scalable dataset for evaluating symbolic search. We identify three fundamental dependency patterns among free variables: \emph{disconnected} queries, \emph{chain} queries, and \emph{cyclic} queries, corresponding to independent, sequential, and mutually constrained variables. Together, these patterns span a broad spectrum of free-variable dependencies. We further enrich these structures with \textbf{multi-edge} and \textbf{negative-edge} variants. Starting from seven two-variable query types in the original $\text{EFO}_k$-CQA benchmark, we extend them to support up to three free variables, resulting in 14 query types that capture increasing structural complexity. Figure~\ref{fig:query types} visualizes the constructed query types.

The results of the HIT$@10$ metric on our new dataset are presented in Table~\ref{tab: EFO3 result}, which provide the following observation:
\begin{enumerate}
\item \textbf{The performance of NS3(J) surpasses existing baseline methods on nearly all 14 query types and 3 knowledge graphs.} The impressive results show the effectiveness of symbolic search pruning, which enhances neural symbolic search methods by reducing complexity. In particular, our proposed NS3(J) is, on average, two to three times better than the previous state-of-the-art method, FIT. This highlights the importance of \textbf{direct joint ranking} and validates the effectiveness of our domain reduction strategy. 
\item \textbf{As the number of free variables increases, the performance of most models tends to decline, particularly for embedding-based methods.} In contrast, our marginal ranking results remain consistently stable, likely due to the robustness of our approach in solving the marginal answer set. However, since marginal ranking can only estimate joint ranking, the final results remain suboptimal compared to our direct joint ranking method, as shown in Table~\ref{tab: EFO3 result}.
\end{enumerate}


\begin{figure*}[ht]
    \centering
    \begin{minipage}{0.33\textwidth}
        \centering
        \includegraphics[width=\linewidth]{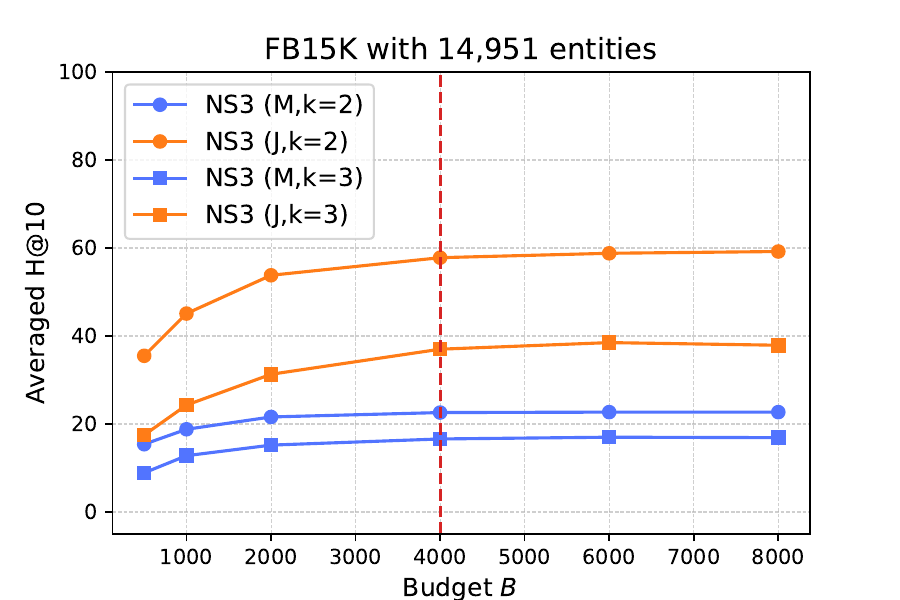}
    \end{minipage}\hfill
    \begin{minipage}{0.33\textwidth}
        \centering
        \includegraphics[width=\linewidth]{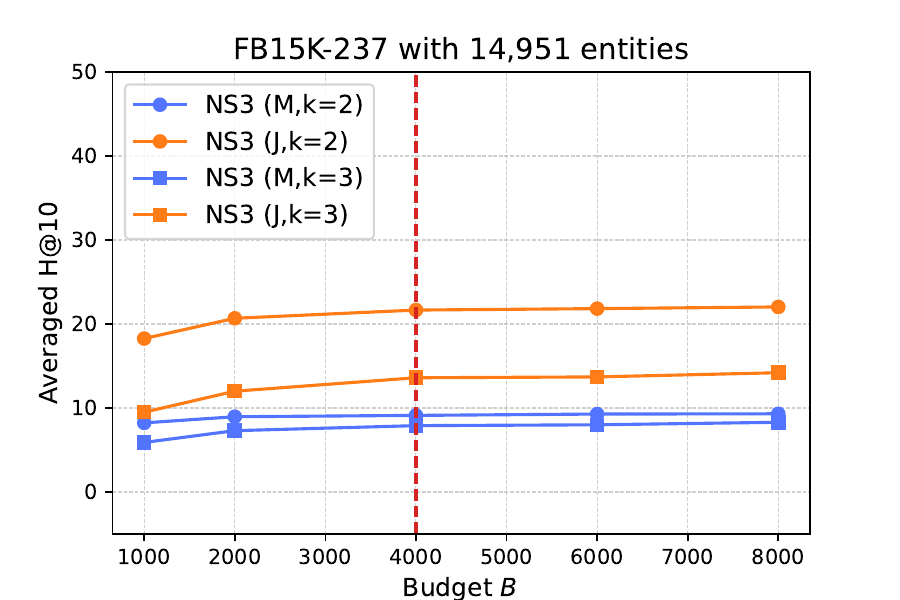}
    \end{minipage}
    \begin{minipage}{0.32\textwidth}
        \centering
        \includegraphics[width=\linewidth]{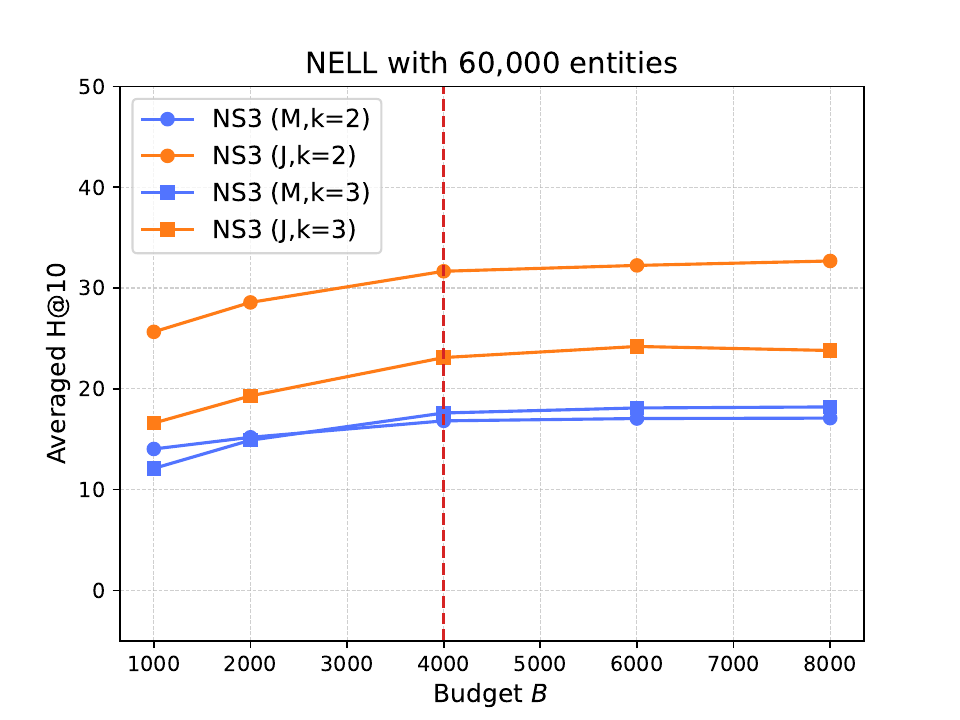}
    \end{minipage}
    \vspace{-1em}

    \caption{The performance of NS3 (J) and NS3 (M) varying with  the domain budget $B$ across three KGs. We provide the averaged HIT@10 scores(\%) for the number of free variables $k$. Particularly, we highlight the domain budget $B$ utilized in Table~\ref{tab: EFO3 result}.}
    \label{fig: ablation on B}
    \vspace{-1em}
\end{figure*}

\subsection{Varying budget Size}~\label{sec:ablation}

Domain budget $B$ is an important hyperparameter because it directly adjusts the trade-off between accuracy and computational complexity. Thus, we conduct the experiments varying $B$ on three KGs, presenting the average scores and total running times on the scalable $\text{EFO}_k$ datasets  in Figure~\ref{fig: ablation on B} and Table~\ref{tab:time_results}, respectively.

\begin{table}[t]
    \centering
    \caption{Total running times of NS3 (J) for inferring 14 query types on scalable $\text{EFO}_k$ datasets, across different datasets and budget levels.}
    \begin{tabular}{ccccc}
        \toprule
        \textbf{Times (hours)} & \textbf{1k} & \textbf{2k} & \textbf{4k} & \textbf{8k} \\
        \midrule
        FB15K-237 & 0.25 & 0.42 & 0.53 & 1.47 \\
        \midrule
        FB15K & 0.5 & 0.5 & 0.85 & 1.97 \\
        \midrule
        NELL & 0.22 & 0.27 & 0.58 & 1.33 \\
        \bottomrule
    \end{tabular}

    \label{tab:time_results}
    \vspace{-1em}
\end{table}

Regarding the total running times of NS3 (J), the growth on three KGs remains stable as the domain budget $B$ increases from 1k to 8k. This demonstrates the stability of the NS3 framework.

Additionally, the average results across various scenarios steadily rise, but the scores gradually begin to level off. This upward trend demonstrates the effectiveness of our budgeting algorithms, which successfully identify valid candidate answers. However, the benefits of further increasing the budget gradually diminish, particularly at the chosen value of $B$ (marked in red). This saturation is consistent with the pruning recall analysis in Appendix~\ref{app:ns3_additional}, where most gold tuples are retained after marginal pruning under the selected budgets. It is worth noting that  our chosen $B$ still constitutes only a small portion of the search space. For instance, the search space for the $\text{EFO}_2$ query from FB15K-237 is actually $(\text{15k})^2$, meaning 8k accounts for only $4\%$. For other KGs or queries with more free variables, this ratio will be even lower, which further highlights the significantly reduced search space.

\subsection{Ablation study of budgeting algorithms}

\begin{table}[ht]
    \centering
    \caption{Ablation study on the utilization of the budgeting algorithm. We report the averaged HIT@10 metrics over queries with $k=3$. }
    \label{tab: ablation budgeting}
    \begin{tabular}{cccc}
    \toprule
        Strategy & FB15K-237 & FB15K & NELL \\ \midrule
        Progressive budget budgeting & 13.72 & 40.22 & 27.34 \\ \midrule
        Simultaneous budget budgeting & 11.58 & 27.56 & 23.96 \\ \bottomrule
    \end{tabular}
    \vspace{-1.5em}

\end{table}
To demonstrate the effectiveness of our progressive budgeting algorithm, we compare it with a strategy that merges all free variables at once. This approach first computes the truth values of marginal queries for each free variable and then allocates the budget simultaneously based on the counts of these marginal queries. We present an ablation study on queries with $k>2$ in Table~\ref{tab: ablation budgeting}, since the two strategies are equivalent when $k=2$.

The experimental results show that our progressive budgeting algorithm outperforms the direct merging strategy across all three KGs, particularly on the FB15K dataset. This highlights the effectiveness of our proposed budgeting algorithm.


\section{Conclusion}


In this paper, we investigate the novel yet challenging problem of answering complex logical queries with $k$ free variables. Building on the proposed \emph{marginal} and \emph{merge} transformations, we enable existing symbolic search methods to directly infer joint rankings for $\text{EFO}_k$ queries for the first time. To address the exponential growth of the joint domain, we introduce a budgeting-based framework that dynamically reduces the joint search space. As a result, complex $\text{EFO}_k$ queries are transformed into equivalent $\text{EFO}_1$ queries and evaluated over the reduced domain. In marginal ranking benchmarks for $\text{EFO}_k$-CQA, our approach to compute marginal rankings through marginal queries—consistently outperforms all baselines in two key metrics. To further validate robustness and scalability, we construct a new scalable dataset with more free variables and evaluate models using joint rankings. The results demonstrate that our method substantially surpasses baseline models. Although our experiments focus on $k \le 3$ due to benchmark availability, NS3 is designed to progressively reduce the joint space, with complexity scaling polynomially in $k$ under a fixed budget. Extending to larger $k$ remains an important direction for future work.

\section{Acknowledgments}
The authors of this paper were supported by the National Key Research and Development Program of China (2025YFE0200500), the ITSP Platform Research Project (ITS/189/23FP) from ITC of Hong Kong, SAR, China, and the AoE (AoE/E-601/24-N), the CRF (No. C6004-25G), the RIF (R6021-20), and the GRF (16205322) from RGC of Hong Kong, SAR, China.
\bibliographystyle{ACM-Reference-Format}
\balance
\bibliography{my_refs}

\appendix

\section{Evaluation and baseline details}~\label{app: baselines}

\subsection{Metrics of marginal ranking}~\label{app: eval}

Complex query answering aims to discover new answers to logical queries over incomplete answers. Consider an  observed knowledge graph $\mathcal{KG}_o$, and a more complete knowledge graph $\mathcal{KG}_c$ with $\mathcal{KG}_o \subset \mathcal{KG}_c$. Given a logical query $\phi$, the easy answers are those that satisfy the query in the observed graph, while hard answers are those that are true in the complete graph but false in the observation graph. In the previous benchmark of EFO1 queries, the evaluation protocol used the ranking-based metrics including MRR and HIT$@k$ to score the ranking of hard answers. When evaluating an EFO$k$ query, the ranked candidates should be the tuples of $k$ entities, with their space represented as $\mathcal{E}^k$. We refer to the ranking of tuple answers over $\mathcal{E}^k$ as joint ranking. The evaluation for multi-variable queries should utilize ranking-based metrics over this joint ranking. 

Due to the limitations of existing CQA models, the current multi-variable benchmark EFO$k$-CQA~\citep{yin25_efok} just infers marginal variables separately and then evaluates the resulted marginal rankings. Given an EFO$k$ query $\phi(y_1, \cdots, y_k)$, the existing CQA models are adapted to compute the marginal rankings of the free variables, where the marginal ranking sorts the entities based on the degree to which they satisfy the constraints of the corresponding free variable. To evaluate the EFO$k$ queries, EFO$k$-CQA proposed three types of metrics that extend ranking-based metrics for EFO1 queries, transitioning from marginal to joint evaluations. Let 
$ \mathbf{a} = \{(a_1,\cdots,a_k)\}$ be the answer set of $\phi(y_1,\cdots,y_k)$, denote $\mathbf{a}_i = M(\mathbf{a}, \{ y_i\})$ be the marginal answer set of $y_i$. We can define the hard marginal answer set $\mathbf{a}^h_i$ and the easy marginal answer set $\mathbf{a}^e_i$. Then, we rank the hard answer $\mathbf{a}^h_i$ against the non-answer $\mathcal{E}-\mathbf{a}^h_i-\mathbf{a}^e_i$ and use the ranking to compute standard metrics like MRR, HIT$@k$ as the marginal results for every free variable. The first \textbf{marginal} metrics directly take the average of the marginal results on all free variables. However, the marginal ranking fails to assess the correspondence among free variables. To address this, \textbf{Multiply} metrics were proposed to evaluate the combinatorial answer based on the marginal ranking. Consider the tuple hard answer $ (a^j_1,\cdots,a^j_k)$, we denote $R_i[a^j_i]$ as the marginal ranking of $a^j_i$ over free variable $f_i$. The \textbf{Multiply} metrics are HITS$@n^k$, where it's 1 if all $a_i$ is ranked in the top $n$ in the corresponding node $y_i$, and $0$ otherwise. The \textbf{Joint} metrics aim to assess the estimated joint rankings by estimating them with the marginal rankings. Given a candidate tuple $ \mathbf{a} = \{(a_1,\cdots,a_k)\}$, we still use  $R_i[a^j_i]$ represent the marginal ranking of $a^j_i$ over free variable $f_i$. We sort the candidate tuples based on their summed ranking $\hat{R}^j = \sum_{i=1}^k r_i(a_i^{j})$ over $\mathcal{E}^k$ to obtain the joint ranking. We expand previous evaluation only supporting two free variables to any positive integer $k$ with combinatorics. The joint ranking of a tuple is represented by the number of non-negative integer solutions, which has a closed-form solution, to the following inequality:
\begin{equation}
    x_1 + \cdots + x_k \le \hat{R}^j.
\end{equation}
Next, we introduce a new variable to simplify the inequality, transforming it into the following linear indeterminate equation:
\begin{align}
        x_1 + \cdots + x_k + x_{k+1} = \hat{R}^j.
\end{align}
Based on combinatorial principles, the number of non-negative integer solutions to this linear indeterminate equation is given by 
\begin{align}
\binom{\hat{R}^j+k}{k} = \frac{\prod_{i=1}^k(\hat{R}^j+i)}{k!}.
\end{align}
Although the joint metrics are designed to evaluate the joint ranking across the entire search space in the EFO$k$-CQA benchmark, these metrics are a compromise for marginal CQA models and provide a closed form based on the marginal rankings.

\subsection{Baseline implementation details}

In this section, we provide implementation details of CQA models that have been evaluated in our paper. 

For query embedding methods that utilize the operator tree, including BetaE~\citep{ren_beta_2020}, LogicE~\citep{luus_logic_2021}, and ConE~\citep{zhang_cone_2021}, we first compute the ordering of nodes in the query graph. Subsequently, we calculate the embedding for each node, with the final embedding of every free node serving as the predicted answer. Notably, the node ordering we derive aligns with the natural topological ordering induced by the directed acyclic operator tree, allowing us to compute embeddings in the same sequence as the original implementation. We then implement each set operation in the operator tree, including intersection, negation, and set projection. By leveraging Disjunctive Normal Form (DNF), we address the union in the final step. Thus, our implementation is consistent with the original dataset~\citep{ren_beta_2020}.

For CQD~\citep{arakelyan_complex_2020} and LMPNN~\citep{wang_logical_2023}, their original implementations do not rely on the operator tree, so we utilize them as is. Specifically, in a query graph with multiple free variables, CQD predicts the answer for each free variable individually, treating the other free variables as existential variables. For LMPNN, we obtain the embeddings of all nodes representing the free variables.

Regarding FIT~\citep{yin_rethinking_2023}, although it is designed to infer $\mathrm{EFO}_1$ queries, it is computationally expensive. Its inference complexity is $O(|\mathcal{E}|^2)$ even for acyclic query graphs, and it fails to remain polynomial for cyclic graphs due to its reliance on explicit enumeration. In our implementation, we therefore restrict FIT to enumerate at most 10 candidate entities per node in the query graph, as exhaustive enumeration for cyclic queries is prohibitively time-consuming. This adjustment has enabled the implementation of FIT on the FB15k-237 dataset~\citep{toutanova_observed_2015}. However, it took 20 hours to evaluate FIT on our EFO$k$-CQA dataset, while other models required no more than two hours. Additionally, for larger knowledge graphs, including NELL~\citep{carlson_toward_2010} and FB15k~\citep{bordes_translating_2013}, we encountered out-of-memory errors on a Tesla V100 GPU with 32GB of memory when implementing FIT. Consequently, we have omitted its results for these two knowledge graphs.

\section{Budgeting and runtime analysis}~\label{app: further exper}

\subsection{Implementation and runtime details}~\label{app: imple}

Our experiments were conducted on a V100 GPU with 32 GB of memory. Evaluating the $\text{EFO}_k$-CQA benchmark on FB15K-237 using FIT requires 20 hours, while our method takes only 5 to 6 hours. Additionally, FIT encounters an out-of-memory error on other knowledge graphs, whereas our method does not. Evaluating FIT on our proposed scalable $\text{EFO}_k$ benchmark on NELL takes 2.8 hours, while our method requires only 0.9 hours.

\subsection{Recall after marginal pruning}

Marginal pruning is a necessary but lossy step: once one component of a gold tuple is excluded from the retained marginal candidates, the tuple cannot be recovered by later symbolic search. We quantify this risk by measuring the fraction of gold tuples whose components are all retained after pruning.

\begin{table}[t]
\centering
\caption{Gold tuple recall after marginal pruning under different budgets.}
\label{tab:pruning_recall}
\resizebox{0.35\textwidth}{!}{
\begin{tabular}{lcccc}
\toprule
Dataset & $B=1000$ & $B=2000$ & $B=4000$ & $B=6000$ \\
\midrule
FB15K-237 & 77.4 & 91.9 & 99.0 & 99.8 \\
FB15K & 76.4 & 91.1 & 98.8 & 99.8 \\
NELL & 75.8 & 86.0 & 93.1 & 95.8 \\
\bottomrule
\end{tabular}
}
\end{table}

As shown in Table~\ref{tab:pruning_recall}, recall exceeds 98\% on FB15K-237 and FB15K at $B=4000$, which explains why performance gains begin to saturate around this budget. NELL has lower recall because its graph is sparser, making cascading pruning errors more frequent.

\subsection{Budget sensitivity and runtime}

The budget controls the trade-off between candidate coverage and inference cost. On NELL, the total inference time increases moderately as $B$ grows, while the performance improvement becomes smaller after the recall is close to saturation.

\begin{table}[t]
\centering
\caption{Runtime of NS3 on NELL under different budgets.}
\label{tab:nell_budget_runtime}
\resizebox{0.35\textwidth}{!}{
\begin{tabular}{lccccc}
\toprule
Budget & 1000 & 2000 & 4000 & 6000 & 8000 \\
\midrule
Time (min) & 12 & 15 & 28 & 40 & 44 \\
\bottomrule
\end{tabular}
}
\end{table}

Together with Figure~\ref{fig: ablation on B}, this suggests that $B=4000$ is a practical operating point on FB15K-237 and FB15K: it retains almost all gold tuples while avoiding unnecessary expansion of the joint domain.

\subsection{Direct runtime comparison}

Table~\ref{tab:runtime_baseline} reports end-to-end evaluation time under the same hardware. ConE is faster because it only performs embedding-based marginal inference. FIT is the main symbolic-search baseline capable of joint reasoning, and NS3 is 2--4$\times$ faster than FIT while avoiding its out-of-memory issue on larger KGs.

\begin{table}[t]
\centering
\caption{End-to-end runtime comparison in minutes.}
\label{tab:runtime_baseline}
\resizebox{0.25\textwidth}{!}{
\begin{tabular}{lccc}
\toprule
Dataset & FIT & NS3 & ConE \\
\midrule
FB15K-237 & 38 & 16 & 1.8 \\
FB15K & 22 & 17 & 2.1 \\
NELL & 164 & 40 & 2.5 \\
\bottomrule
\end{tabular}
}
\end{table}

\subsection{Ablation study of budget scheduling algorithms}

\begin{table}[t]
    \centering
    \caption{Ablation study on the utilization of the budget scheduling algorithm. We report the averaged HIT@10 metrics over queries with $k=3$. }
    \label{tab: ablation scheduling}
    \begin{tabular}{cccc}
    \toprule
        Method & FB15K-237 & FB15K & NELL \\ \midrule
        With scheduling algorithm & 13.72 & 40.22 & 27.34 \\ \midrule
        Without scheduling algorithm & 11.58 & 27.56 & 23.96 \\ \bottomrule
    \end{tabular}
 \vspace{-1em}

\end{table}

We conduct an ablation study in Table~\ref{tab: ablation scheduling} to demonstrate the effectiveness of our proposed budget scheduling algorithm. We compare the budget scheduling algorithm to a method that directly merges all the free variable nodes based on their marginal answer sets. We present the averaged results for queries with $k>2$, as the two methods yield equivalent results for queries with $k=2$. The findings indicate that our budget scheduling algorithm significantly improves performance compared to direct merging, particularly on the FB15K dataset. This underscores the effectiveness of our proposed scheduling algorithm.

\subsection{Further results on FB15K}

We present the benchmark results for FB15K in Table~\ref{tab: FB15k EFO2 result}. Notably, FIT encounters out-of-memory issues due to the large knowledge graph. Our proposed NS3 (M) achieves the best average results across three different metrics.

\begin{table}[h]
\centering
\caption{HIT@10 scores(\%) of three different types for answering queries with two free variables on FB15k. The constant number is fixed to be two. The notation of $e$, SDAG, Multi, and Cyclic is the same as Table~\ref{tab: EFO2 result}.}
\label{tab: FB15k EFO2 result}
\resizebox{\linewidth}{!}{
\begin{tabular}{ccrrrrrrrrr}
\toprule
\multirow{2}{*}{Model}  & \multirow{2}{*}{\shortstack[c]{HIT@10\\ Type}}  & \multicolumn{2}{c}{$e=0$} & \multicolumn{3}{c}{$e=1$} & \multicolumn{3}{c}{$e=2$} & \multirow{2}{*}{AVG.} \\
\cmidrule(lr){3-4} \cmidrule(lr){5-7} \cmidrule(lr){8-10} & & SDAG      & Multi &  SDAG  & Multi & Cyclic &  SDAG  & Multi & Cyclic & \\
\midrule
\multirow{3}{*}{BetaE} 
& Marginal&76.9&77.2&68.9&69.3&75.1&55.0&57.4&73.6&63.6 \\
& Multiply&41.7&41.6&31.7&31.0&38.7&25.2&25.9&36.1&29.7\\
& Joint&11.6&13.7&8.7&8.6&17.8&4.9&5.4&14.3&8.4\\
\midrule
\multirow{3}{*}{LogicE} 
& Marginal&82.9&80.9&73.6&72.9&76.6&58.9&60.7&75.7&66.9 \\
& Multiply&47.5&45.0&36.3&34.1&40.4&28.5&29.0&38.0&32.7\\
& Joint&12.7&13.9&10.0&9.9&19.2&6.1&6.5&15.9&9.6 \\
\midrule
\multirow{3}{*}{ConE} 
& Marginal&84.1&84.8&76.5&76.3&81.4&61.8&63.8&79.7&70.2 \\
& Multiply&48.7&48.1&37.7&35.9&44.2&29.9&30.4&41.4&34.6\\
& Joint&14.2&15.6&10.3&10.4&20.6&6.2&6.6&16.9&10.1 \\
\midrule
\multirow{3}{*}{CQD} 
& Marginal&73.8&76.8&69.0&71.9&76.3&51.1&54.4&77.0&62.9 \\
& Multiply  &45.0&46.6&37.4&36.9&43.9&28.1&29.2&41.9&34.0\\
& Joint   &17.1&19.0&13.1&13.0&20.6&7.7&8.6&18.1&11.9  \\
\midrule
\multirow{3}{*}{LMPNN} 
& Marginal&89.2&80.1&80.3&78.2&84.2&65.6&63.7&80.2&71.3 \\
& Multiply   &56.6&50.5&45.7&42.4&49.0&37.6&34.8&44.6&39.7\\
& Joint   &18.9&17.2&12.9&12.4&22.4&8.0&7.5&16.9&11.2  \\
\midrule

\multirow{3}{*}{NS3 (M)} 
& Marginal&82.0&83.6&81.6&83.4&55.5&76.7&78.9&82.2&79.1 \\
& Multiply   &81.8&83.1&81.1&82.7&44.4&78.0&79.3&46.6&78.5\\
& Joint   &19.9&24.2&21.1&23.5&19.7&13.0&14.3&18.5&18.5  \\
\bottomrule
\end{tabular}
}
\end{table}

\section{Comparison with multi-modal $\mathrm{EFO}_1$ methods}\label{app:ns3_additional}

Recent multi-modal query embedding methods, such as Query2Particles and Query2GMM, improve $\mathrm{EFO}_1$ reasoning by representing the answer set of a single free variable with multiple particles or mixture components. NS3 addresses a different problem: $\mathrm{EFO}_k$ joint ranking, where the output is a ranked list of entity tuples coupled by cross-variable constraints. Applying a multi-modal $\mathrm{EFO}_1$ method independently to each free variable still produces marginal rankings and cannot determine tuple compatibility. In contrast, NS3 uses marginalized queries only to prune candidates, then merges free variables into hypernodes and performs symbolic search over joint assignments. Therefore, NS3 is complementary to stronger $\mathrm{EFO}_1$ solvers and can use them as backbones.

\section{Proof of the commutativity of marginalization}~\label{app: proof}

\begin{proposition}[Commutativity of marginalization]
    Given a query $\phi$ and a set of target free variables $\mathcal{Y}^t$, 
    \begin{equation}
         \mathcal{A}[\mathcal{M}[\phi(\mathcal{Y}), \mathcal{Y}^t ]] = \mathcal{M}[\mathcal{A}[\phi(\mathcal{Y})], \mathcal{Y}^t ].
    \end{equation}
\end{proposition}

\begin{proof}
Let $\hat{\mathbf{a}} \in M[\mathcal{A}[\phi(\mathcal{Y})], \mathcal{Y}^m]$. Then, there exists an answer $\mathbf{a} \in \mathcal{A}[\phi(\mathcal{Y})]$ such that $\hat{\mathbf{a}} = M[\mathbf{a}, \mathcal{Y}^m]$. Since $\mathbf{a}$ is a valid assignment for the query $\phi$, we can treat the entities in $\mathcal{Y} \setminus \mathcal{Y}^m$ as assignments for existential variables. This implies that $M[\mathbf{a}, \mathcal{Y}^m]$ is the result of the marginal query $M[\phi, \mathcal{Y}^m]$. Thus, we have \( M[\mathcal{A}[\phi(\mathcal{Y})], \mathcal{Y}^m] \subset \mathcal{A}[M[\phi(\mathcal{Y}), \mathcal{Y}^m]] \).

Now, consider $\tilde{\mathbf{a}} \in \mathcal{A}[M[\phi(\mathcal{Y}), \mathcal{Y}^m]]$. For each existential variable transformed from the free variable set $\mathcal{Y} \setminus \mathcal{Y}^m$, we can find at least one corresponding assignment for the given answer $\tilde{\mathbf{a}}$ of $M[\phi(\mathcal{Y}), \mathcal{Y}^m]$. We can then extend $\tilde{\mathbf{a}}$ to $\mathbf{a}$ by assigning values to the existential variables, where $\mathbf{a} \in \mathcal{A}[\phi(\mathcal{Y})]$. Thus, we conclude that $\tilde{\mathbf{a}} \in M[\mathbf{a}, \mathcal{Y}^m]$ and $\tilde{\mathbf{a}} \in M[\mathcal{A}[\phi(\mathcal{Y})], \mathcal{Y}^m]$. Therefore, it follows that \( \mathcal{A}[M[\phi(\mathcal{Y}), \mathcal{Y}^m]] \subset M[\mathcal{A}[\phi(\mathcal{Y})], \mathcal{Y}^m] \).

We conclude the proof.

\end{proof}

\end{document}